\documentclass{article}
\usepackage{arxiv}
\usepackage[utf8]{inputenc} 
\usepackage[T1]{fontenc}    
\usepackage{hyperref}       
\usepackage{url}            
\usepackage{booktabs}       
\usepackage{amsfonts}       
\usepackage{nicefrac}       
\usepackage{microtype}      
\usepackage{lipsum}
\usepackage{amsthm}
\usepackage[ruled,vlined]{algorithm2e}
\usepackage{graphicx}
\usepackage{natbib}
\setcitestyle{authoryear,open={(},close={)}}
\newtheorem{theorem}{Theorem}[section]

\newtheorem{lemma}[theorem]{Lemma}
\usepackage{booktabs}
\usepackage{float}

\title{A boosted outlier detection method based on the spectrum of the Laplacian matrix of a graph}

\author{
  Nicolas Cofre \\
  \texttt{nicolas.cofre@gatech.edu}\\
}

\begin{document}
\maketitle

\begin{abstract}
This paper explores a new outlier detection algorithm based on the spectrum of the Laplacian matrix of a graph. Taking advantage of boosting together with sparse-data based learners. The sparcity of the Laplacian matrix significantly decreases the computational burden, enabling a spectrum based outlier detection method to be applied to larger datasets compared to spectral clustering. The method is competitive on synthetic datasets with commonly used outlier detection algorithms like Isolation Forest and Local Outlier Factor.

\end{abstract}


\section{Introduction}

Spectral clustering ability to discover non-convex clustering structures has been well researched, \cite{von2007tutorial} offers a survey on spectral clustering. Its usage as an outlier detection method based on connectivity has been explored in works like \cite{10.1145/3175684.3175716}. 

In this paper, we explore a new algorithm that exploits boosting and the sparcity of the Laplacian matrix to speed up the computation of the eigenvectors and apply a family of learners based on the spectrum to the outlier detection task.

We show that the performance is comparable or superior to well known outlier detection methods such as Local Outlier Factor from \cite{breunig2000lof} and Isolation Forest from \cite{liu2012isolation}. LOF and Isolation Forest are two widely applied outlier detection algorithms, as stated in \cite{10.1145/3338840.3355641}. Here we show as well how LOF can fail in complex data structures given its main focus on local relations, and how Isolation Forest misses other local structures given its focus in global outliers.

\cite{schapire2003boosting} offers a review on boosting methods, one main difference between our boosted application with respect to other boosting methods like Adaboost is that a proportion of the dataset is removed and the remaining data is passed on to the next learner, progressively decreasing the computational effort needed.

\section{The boosted algorithm}

In this section we present the boosted method, called Boosted Spectral Outlier Detection (BSOD) algorithm. Because boosting methods progressively focus on harder to classify instances, and given that outliers generally represent a very small subset of the whole dataset, we can use a boosting approach to progressively focus on distinctive observations in the dataset.

The $\varepsilon$-neighborhood graph $W$ is defined as the graph in which points within a distance $\varepsilon$ from each other are connected. We have used the Euclidean distance to construct this graph.

The Laplacian matrix $L$ is defined as 
$$L= D-W$$ Where $D$ is the diagonal matrix of degrees i.e. each element in the diagonal is $$d_i = \sum_j w_{ij}$$

$\lambda_k$ corresponds to the $k$-th largest eigenvalue of $L$ and $v_k$ is the associated eigenvector.

\begin{algorithm}[H]
\SetAlgoLined
\KwResult{set of outliers $X^*$ }
\SetKwInOut{Input}{input}
\Input{ $X_0=\{x_i\}_{i=1}^n$, $0<c<1$, $\varepsilon>0$}
 \While{$\#X_i >nc$}{
  $X_i \leftarrow (X_i-\bar{X_i})/s_{X_i}$
  
  $W \leftarrow \varepsilon$-neighborhood graph
  
  $L \leftarrow D-W$
  
  $v_n^i \leftarrow |v^i_n|$
  
  $C_1, C_2 \leftarrow $ k-means on $v_n^i$ with $k=2$
  
  $X_i = \{x_i \in C_j | \#C_j=\max\{\#C_1,\#C_2\}$
%
}
\caption{Boosted spectral outlier detection algorithm, BSOD}
\end{algorithm}

\begin{lemma}
Given a contamination parameter $c<1$ and observations $\{x_i\}_{i=1}^n \subset X$, the boosted algorithm requires a finite number of weak learners to reach $c$
\end{lemma}

\begin{proof}
Given that k-means with $k=2$ is used in the transformed space and that the next learner focuses in only one of those clusters, we have that $X_i \subset X_{i-1}$. Therefore, $\forall i, \#X_i < \#X_{i-1}$   $\Rightarrow \exists i^*, \#X_{i} < nc, \forall i > i^*$.
\end{proof}

\section{Datasets and experiments}

We have tested our algorithm in 2 synthetic datasets. In order to compare the performance, we also test the perfomance on the same datasets for Local Outlier Factor and Isolation Forest. We have used 20 as the number of neighbors for Local Outlier Factor, the default parameter in the Scikit-learn implementation. $100$ trees for Isolation Forest, which is also the default parameter in the Scikit-learn implementation. Finally, $\varepsilon = 0.5$ for all the experiments in BSOD. Each method was given the real contamination value, and the values of precision and recall were recorded for different levels of contamination.

The synthetic dataset 1, corresponds to a circle corresponding to inliners and uniformly random noise as outliers. We have used 4 different levels of contamination. The number of inliners used was 10,000 and a number of outliers was added in order to reach each contamination level. The synthetic dataset 1 is shown below for a contamination level of $10\%$.

\begin{figure}[H]
\caption{Synthetic dataset 1 and real labels}
\centering
\includegraphics[width=0.5\textwidth]{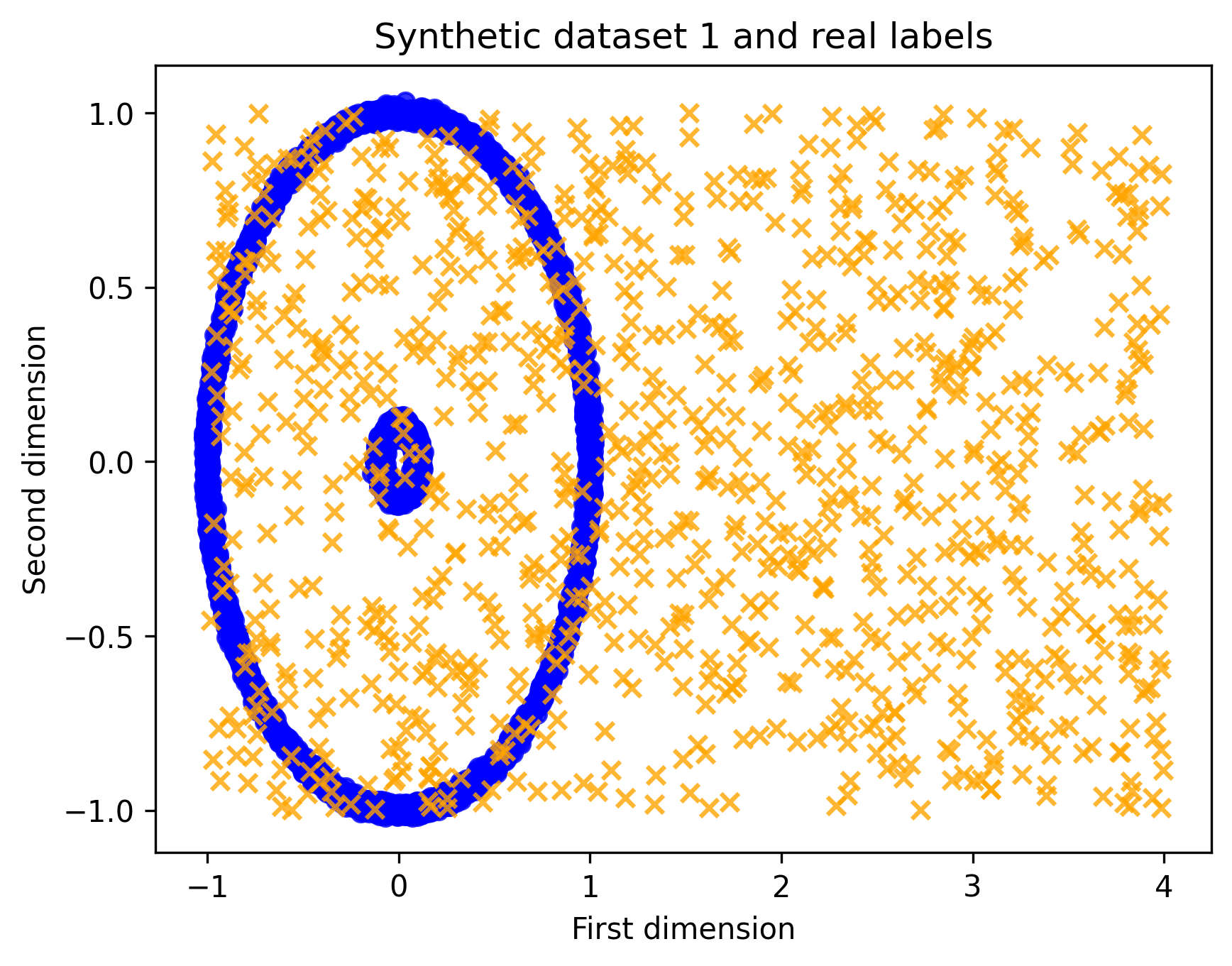}
\end{figure}

The synthetic dataset 2 includes two moons corresponding to inliners and uniformly random noise as outliers. As for the synthetic dataset 1, we have also used 4 different levels of contamination. The number of inliners used was 10,000 and a number of outliers was added in order to reach each contamination level. The synthetic dataset 2 is shown below for a contamination level of $10\%$.

\begin{figure}[H]
\caption{Synthetic dataset 2 and real labels}
\centering
\includegraphics[width=0.5\textwidth]{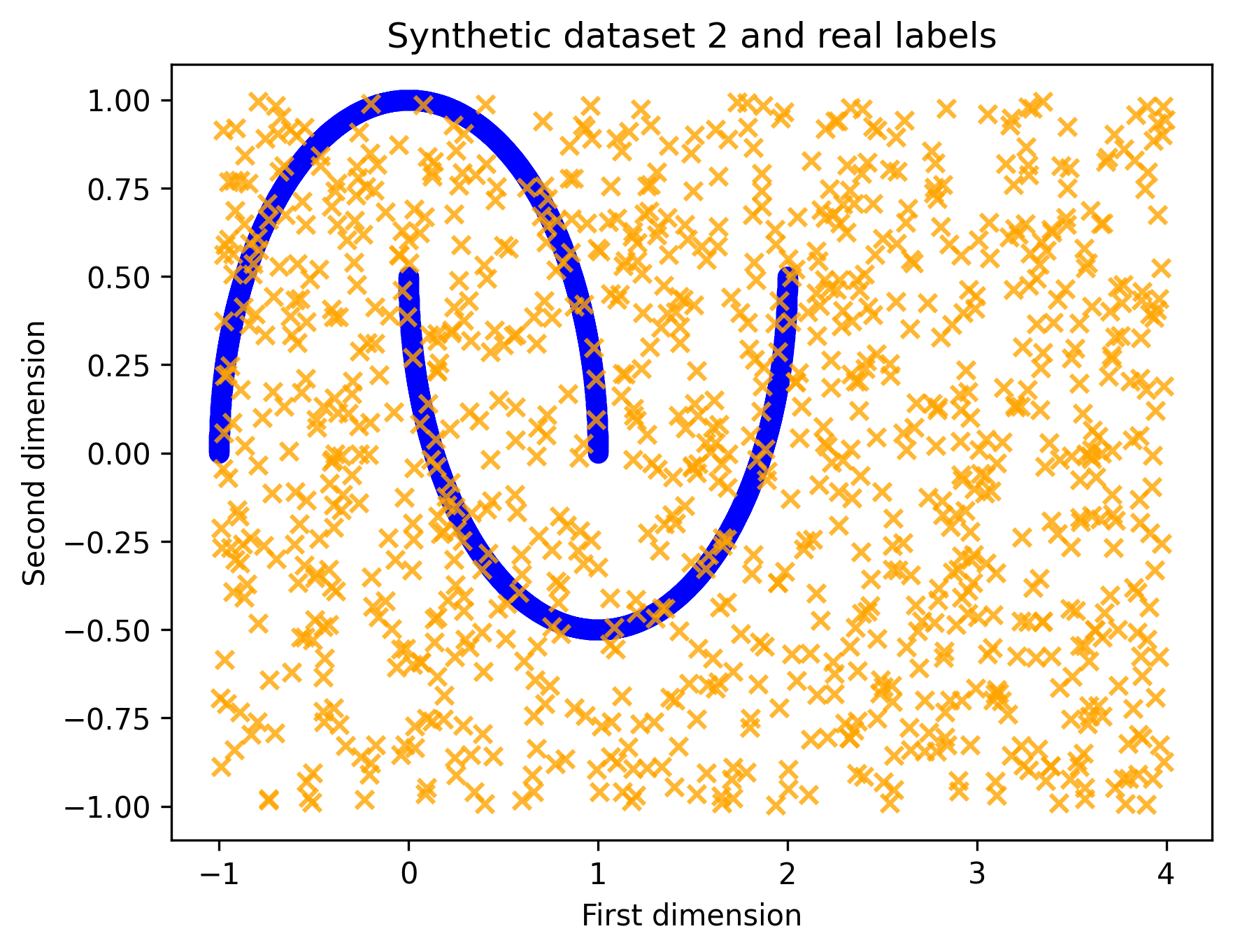}
\end{figure}

\section{Performance analysis on synthetic datasets}
\subsection{Synthetic dataset 1}
In the plots below, we can see the results of the outlier detection methods for the contamination case of $10\%$. We can see how Isolation forest misses the outliers within the 2 circles and that LOF misses the outliers outside of the circle, given that they have similar local density.

\begin{figure}[H]
\caption{BSOD on synthetic dataset 1}
\centering
\includegraphics[width=0.5\textwidth]{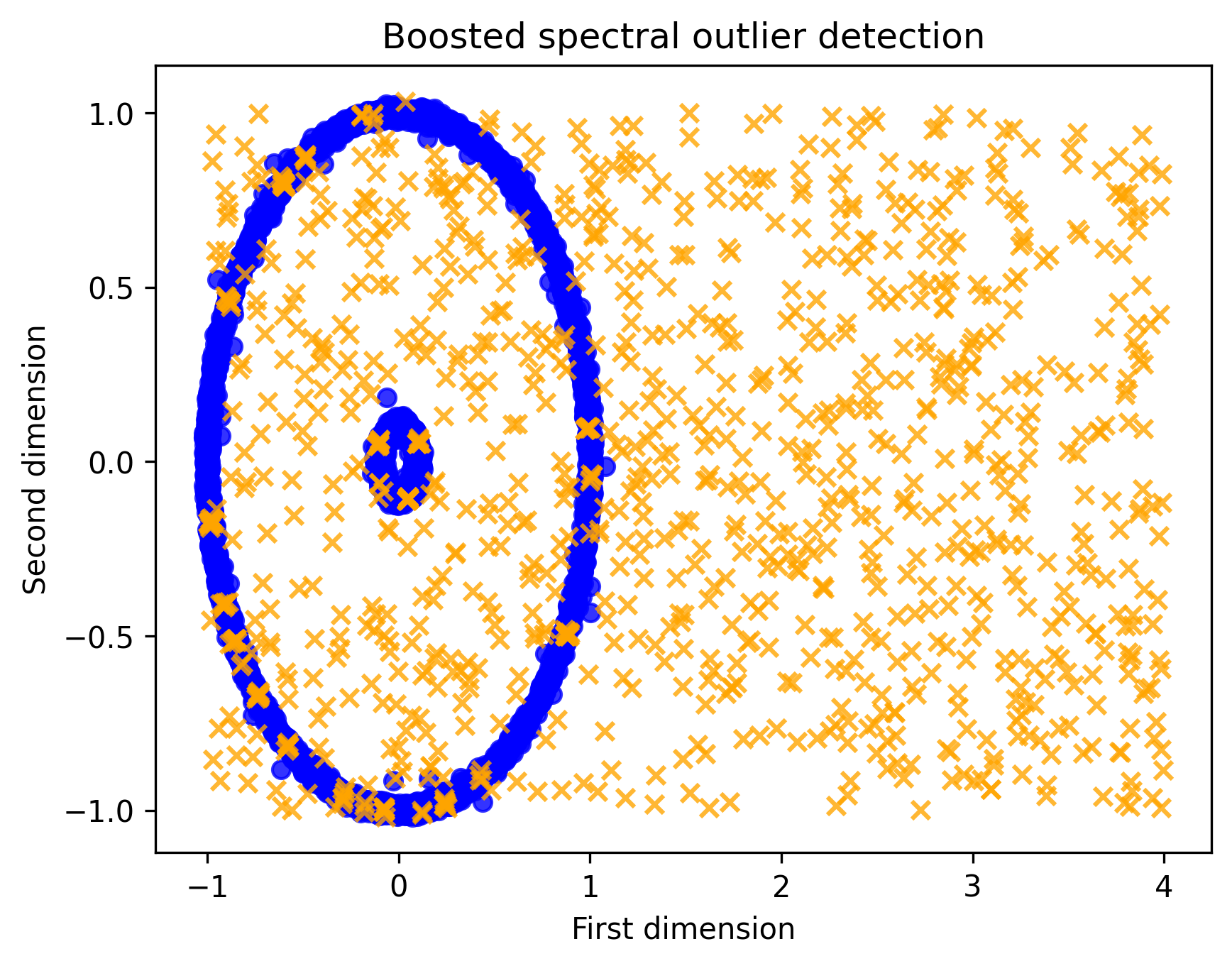}
\end{figure}

\begin{figure}[H]
\caption{Isolation Forest on synthetic dataset 1}
\centering
\includegraphics[width=0.5\textwidth]{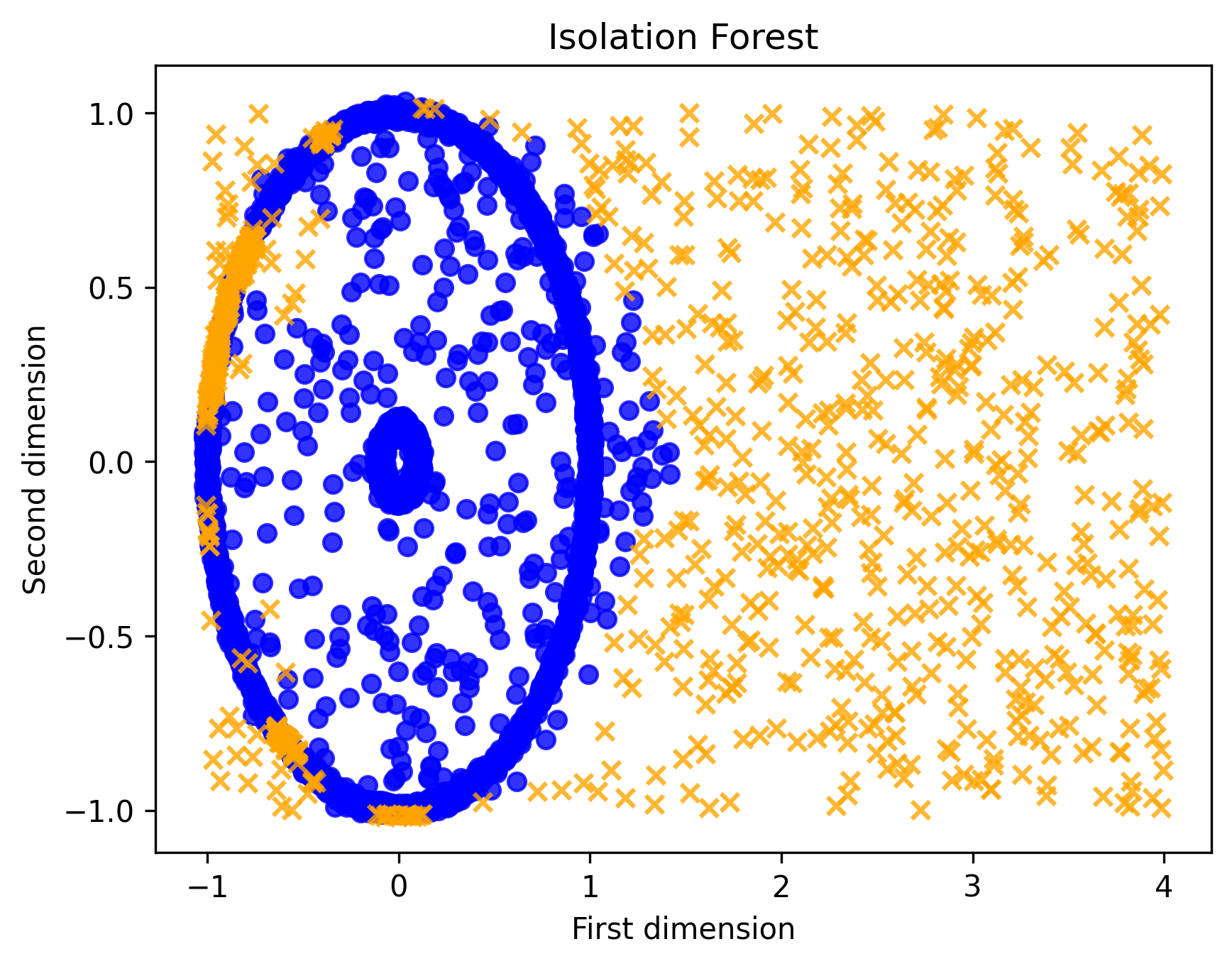}
\end{figure}

\begin{figure}[H]
\caption{Local Outlier Factor on synthetic dataset 1}
\centering
\includegraphics[width=0.5\textwidth]{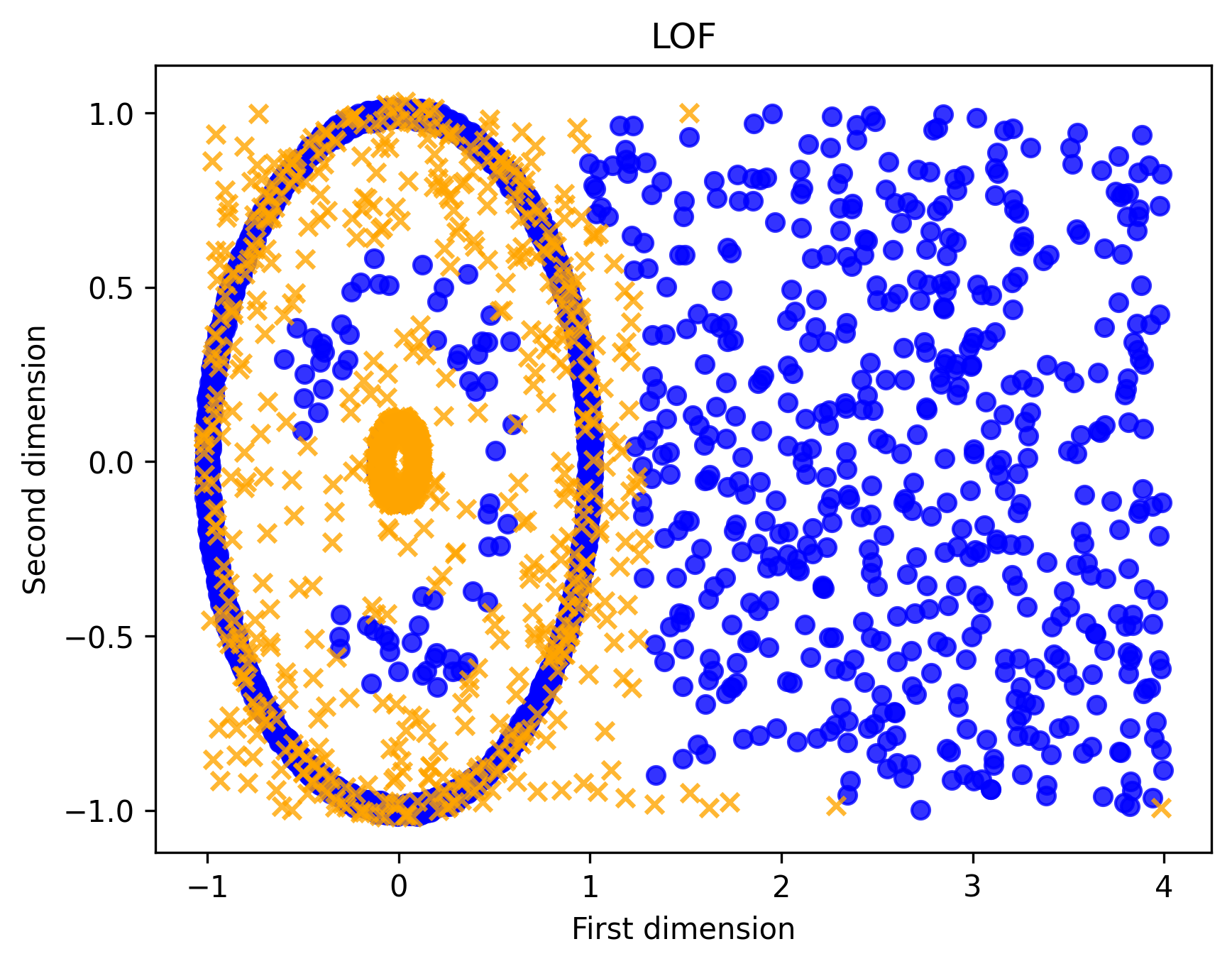}
\end{figure}

The summary, including precision and recall for different contamination levels is below

\begin{table}[H]
\caption{Results on synthetic dataset 1}

\begin{center}
\begin{tabular}{@{}lllllllll@{}}
\toprule
     & \multicolumn{2}{l}{c = 1\%} & \multicolumn{2}{l}{c = 5\%} & \multicolumn{2}{l}{c = 10\%} & \multicolumn{2}{l}{c = 15\%} \\ \midrule
     & Precision      & Recall     & Precision      & Recall     & Precision      & Recall      & Precision      & Recall      \\
BSOD & 0.78           & 0.75       & 0.85           & 0.87       & 0.85           & 0.93        & 0.81           & 0.93        \\
IF   & 0.70           & 0.71       & 0.60           & 0.63       & 0.55           & 0.61        & 0.57           & 0.65        \\
LOF  & 0.50           & 0.51       & 0.41           & 0.43       & 0.31           & 0.34        & 0.27           & 0.31        \\ \bottomrule
\end{tabular}
\end{center}

\end{table}

\subsection{Synthetic dataset 2}
In the plots below, we can see the results of the outlier detection methods for the contamination case of $10\%$. We can see how Isolation forest misses the outliers within the 2 circles and that LOF misses the outliers outside of the circle, given that they have similar local density.

\begin{figure}[H]
\caption{BSOD on synthetic dataset 2}
\centering
\includegraphics[width=0.5\textwidth]{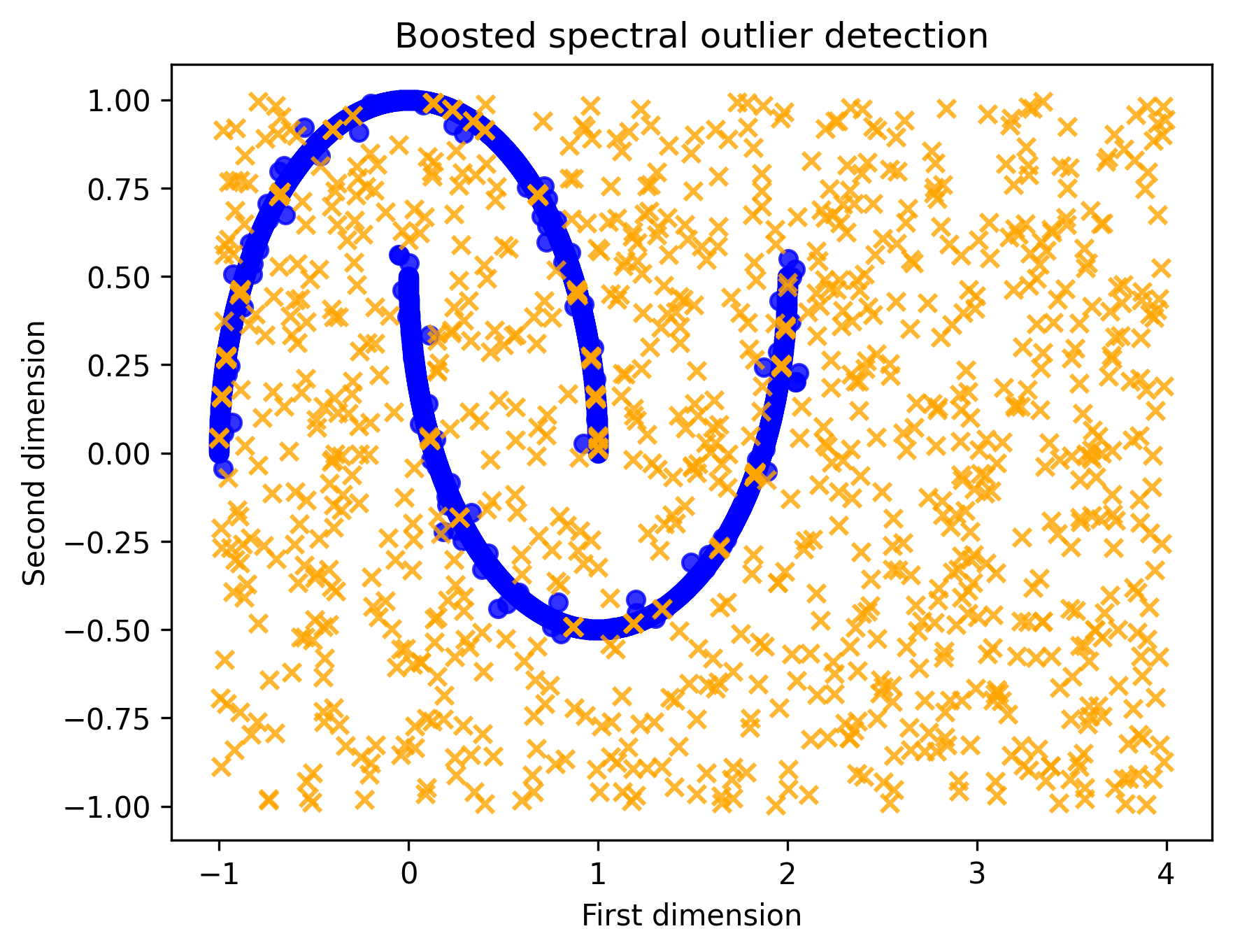}
\end{figure}

\begin{figure}[H]
\caption{Isolation Forest on synthetic dataset 2}
\centering
\includegraphics[width=0.5\textwidth]{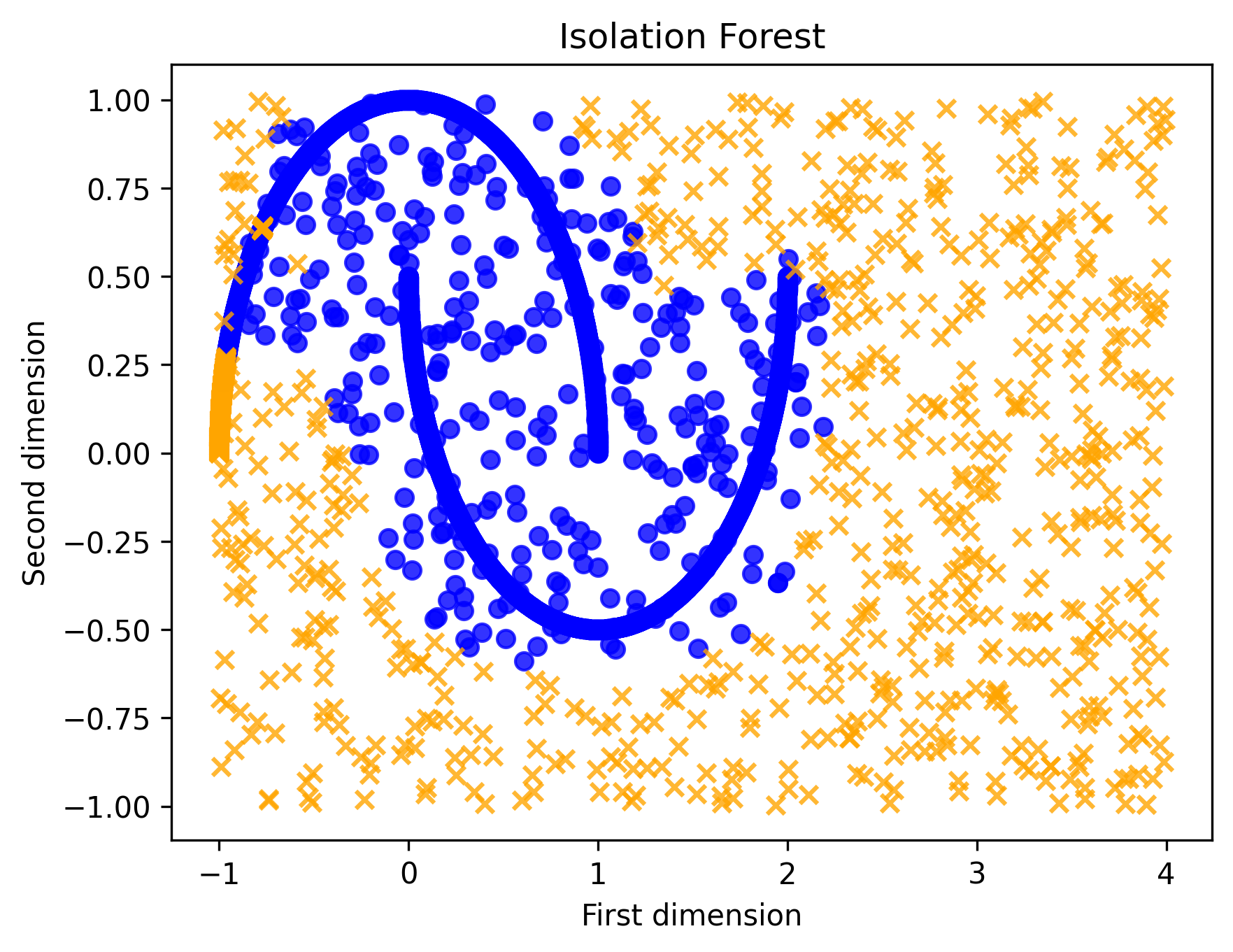}
\end{figure}

\begin{figure}[H]
\caption{Local Outlier Factor on synthetic dataset 2}
\centering
\includegraphics[width=0.5\textwidth]{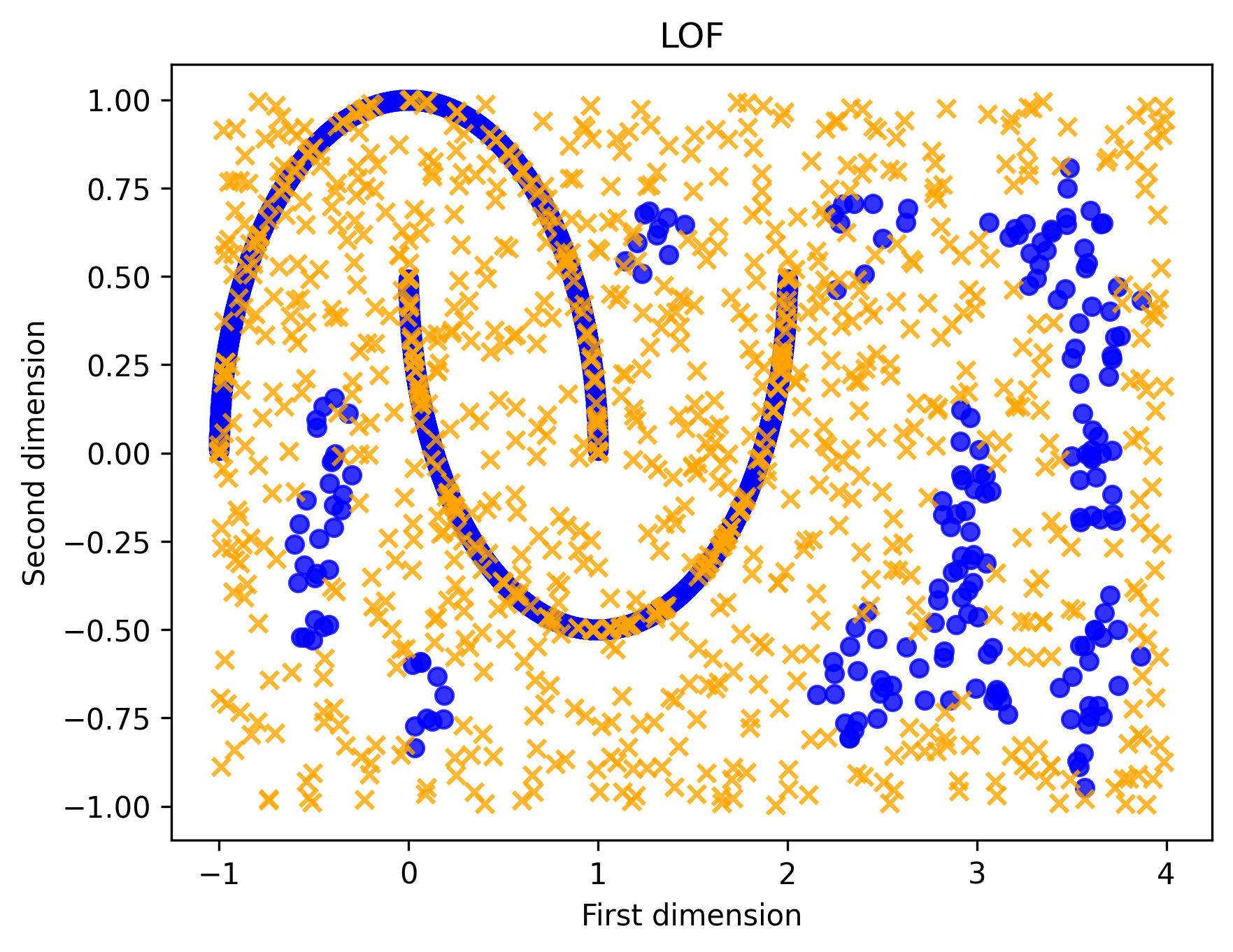}
\end{figure}

The summary, including precision and recall for different contamination levels is below

\begin{table}[H]
\caption{Results on synthetic dataset 2}

\begin{center}
\begin{tabular}{@{}lllllllll@{}}
\toprule
     & \multicolumn{2}{l}{c = 1\%} & \multicolumn{2}{l}{c = 5\%} & \multicolumn{2}{l}{c = 10\%} & \multicolumn{2}{l}{c = 15\%} \\ \midrule
     & Precision      & Recall     & Precision      & Recall     & Precision      & Recall      & Precision      & Recall      \\
BSOD & 0.81           & 0.71       & 0.86           & 0.87       & 0.84           & 0.92        & 0.80           & 0.92        \\
IF   & 0.46           & 0.46       & 0.55           & 0.58       & 0.56           & 0.62        & 0.56           & 0.62        \\
LOF  & 0.89           & 0.90       & 0.80           & 0.84       & 0.70           & 0.78        & 0.66           & 0.76        \\ \bottomrule
\end{tabular}
\end{center}

\end{table}

\section{Conclusions}

In this paper we have introduced a new outlier detection algorithm based on the spectrum of the Laplacian matrix. Our method seems to be robust to both local and global outliers, mainly because it is based on the notion of connectivity in a graph. Our method has some similarities to spectral clustering, like the application of k-means on the transformed space but our method is based on the largest eigenvalue. The boosted application has a nice computational property, the complexity of each learner is decreasing given the fact that the sample size is progressively decreasing. This fact together with using a sparse representation of the Laplacian matrix allows the application of this method to bigger datasets compared to other methods based directly on spectral clustering. Our method obtains superior precision and recall measures for different contamination parameters in the synthetic datasets, which provides some evidence that this method can be competitive with respect to widely applied outlier detection algorithms.

\bibliographystyle{unsrt}  
\bibliography{references}  

\end{document}